\documentclass{article}
\usepackage{url}
\usepackage{authblk}
\usepackage{bm}
\usepackage{verbatim}
\usepackage{algorithm}  
\usepackage{subfigure}
\usepackage{algpseudocode}  
\usepackage{amsmath}  
\usepackage{amsthm}
\usepackage{amssymb}
\usepackage{graphicx}
\usepackage{booktabs}
\usepackage{xcolor}
\usepackage{tikz}
\usepackage{mathtools}
\usepackage{extarrows}
\usepackage[margin=0.92in]{geometry}    
\newtheorem{theorem}{Theorem}[section]
\newtheorem{cor}[theorem]{Corollary}
\newtheorem{defi}[theorem]{Definition}
\newtheorem{lemma}[theorem]{Lemma}
\newtheorem{prop}[theorem]{Proposition}
\newtheorem{assumption}[theorem]{Assumption}
\newtheorem{remark}[theorem]{Remark}

\title{Universality of parametric Coupling Flows over parametric diffeomorphism}

\author[1]{Junlong Lyu  \thanks{Equal contribution}}
\author[1]{Zhitang Chen \thanks{Equal contribution}} 
\author[1]{Chang Feng}
\author[1]{Wenjing Cun}
\author[1]{Shengyu Zhu}
\author[1]{Yanhui Geng}
\author[2]{Zhijie Xu}
\author[2]{Yongwei Chen}

\affil[1]{Huawei Noah's Ark Lab}

\affil[2]{Huawei}

\begin{document}

\maketitle

\begin{abstract}
    Invertible neural networks based on Coupling Flows (\texttt{CFlows}) have various applications such as image synthesis and data compression. The approximation universality for \texttt{CFlows} is of paramount importance to ensure the model expressiveness. In this paper, we prove that \texttt{CFlows} can approximate any diffeomorphism in $C^k$-norm if its layers can approximate certain single-coordinate transforms. Specifically, we derive that a composition of affine coupling layers and invertible linear transforms achieves this universality. Furthermore, in parametric cases where the diffeomorphism depends on some extra parameters, we prove the corresponding approximation theorems for our proposed parametric coupling flows named  \texttt{Para-CFlows}. In practice, we apply \texttt{Para-CFlows} as a neural surrogate model in contextual Bayesian optimization tasks, to demonstrate its superiority over other neural surrogate models in terms of optimization performance.
\end{abstract}

\section{Introduction}
Invertible neural networks (INNs) such as coupling flows are firstly introduced as a class of generative models with a tractable likelihood \cite{Dinh2017DensityEU,Kingma2018GlowGF,Oord2018ParallelWF}, and have shown their usefulness and powerfulness in various machine learning tasks such as inverse problems \cite{Ardizzone2019AnalyzingIP}, probabilistic inference \cite{Louizos2017MultiplicativeNF} and feature extraction \cite{Izmailov2020SemiSupervisedLW} in recent years.  

With plenty of successful applications of INNs, one would wonder if such a type of models have the universal expressiveness. As most generative models mainly concern about the transform between distributions, existing works such as \cite{Huang2018NeuralAF,Jaini2019SumofSquaresPF}  focused on the expressiveness from the distribution perspective. However,  the expressiveness from the distribution perspective does not result in the expressiveness from the mapping perspective, as there are a large (or even infinite) number of
diffeomorphisms mapping the given source $\mu$ to the given target $\nu$ . In many applications, knowing the distributional universality is not yet enough, one may be interested in knowing if the optimal transport \cite{Villani2008OptimalTO}, which finds emerging applications in many fields, e.g., machine learning \cite{Peyr2019ComputationalOT}, wireless communication \cite{Mozaffari2017WirelessCU} and economics \cite{Galichon2016OptimalTM},   can be approximated by invertible neural networks.


Therefore, beyond the distributional universality, it is also important to investigate the universality from the mapping perspective. As INNs are always differentiable and invertible, the mappings generated by INNs are always diffeomorphisms. Diffeomorphism plays an important role in mathematics, physics and engineering domains with applications in fluid dynamics \cite{Ebin1970GroupsOD}, wave propagation \cite{Griffiths1992NumericalWP}, robot controls \cite{Rimon1989TheCO} etc.
A natural question comes to the surface: can all diffeomorphisms be approximated by INNs?  Besides, whether the INNs approximate the derivatives in the meanwhile is also interesting and important, e.g., it provides theoretical guarantee for black box optimization tasks based on surrogate models when gradients are utilized.

More importantly, in many real-world problems, e.g.,  3D Euclidean groups, ODE systems and invertible PDE systems with time $t$ as its parameter, diffeomorphisms are usually described as a parametric type:  a parametric diffeomorphism is a function $f(\bm{y},\bm{x}): \mathbb{R}^{m+d} \to \mathbb{R}^d$ such that for any fixed parameter $\bm{y}_0 \in \mathbb{R}^m$, $f(\bm{y}_0, \cdot)$ is a diffeomorphism between $\mathbb{R}^d$ and $\mathbb{R}^d$.   It is interesting to know if these parametric diffeomorphisms can also be approximated by INNs with additional inputs of parameters. 

Existing works including \cite{Teshima2020CouplingbasedIN} have been proposed to investigate the universality of INNs over diffeomorphisms, however it is not able to handle $C^{k}$-diffeomorphisms when $k = d+1$ with $d$ as the dimension of the space. In this paper, we address this limitation and more importantly, we generalize this structural theorem to the parametric diffeomorphism situation and provide a complete proof. This structural theorem shows that any diffeomorphisms (parametric or not) satisfying compactness conditions can be decomposed into finitely many compositions of single-coordinate transforms, which we can easily prove  to be $C^k$-approximated by, e.g., affine coupling flows together with a single zero-padding. 

{Our contribution is three-fold:}
\begin{itemize}

    \item We improve the approximation results in \cite{Teshima2020CouplingbasedIN} to higher-order derivatives approximation with a simpler proof.
    
    \item We generalize  coupling flows to parametric coupling flows and prove their $C^k$-universal approximation to parametric diffeomorphisms.

    \item We propose a practical neural network structure of the parametric affine coupling flow and demonstrate the advantage of using such a surrogate model for contextual Bayesian Optimization (BO) tasks.
\end{itemize}

\section{Preliminary}\label{preliminary}
In this section, we introduce some prior knowledge on diffeomorphism, parametric diffeomorphism, INNs and universality, as well as existing works on universality of INNs.

In what follows, we always assume $k,m,n,d \in \mathbb{N}^+$, where $k$ represents the derivative order, and $m,n,d$ represent dimensionalities respectively for some Euclidean spaces. Unless otherwise stated, we suppose all vectors are row vectors.
\subsection{$C^k$-diffeomorphism groups on $\mathbb{R}^d$}

{\bf $C^k$-diffeomorphisms.} Consider an invertible map $f$ from $\mathbb{R}^d$ to $\mathbb{R}^d$. $f$ is said to be a $C^k$-diffeomorphism, if $f$ has up to $k$-th continuous derivatives, and $\det| Df(\bm{x})|  \neq 0$ for all $\bm{x}$ . Here $Df$ is the Jaccobian matrix of $f$.

{\bf  $C^k$-diffeomorphism group.} One can easily verify that, given $f, g: C^k$ diffeomorphisms from $\mathbb{R}^d$ to $\mathbb{R}^d$,
$f\circ g: f\circ g(\bm{x}) = f(g(\bm{x}))$ and $f^{-1}$ is still a $C^k$-diffeomorphism from $\mathbb{R}^d$ to $\mathbb{R}^d$.  From all above, if we denote 
$$\text{Diff}^k(\mathbb{R}^d) \triangleq\{ f: f\text{ is a }C^k\text{ diffeomorphism over }\mathbb{R}^d \},   $$
we see that  $\text{Diff}^k(\mathbb{R}^d)$ has a natural group structure with composition as its group operator. 

{\bf Compactly supported $C^k$-diffeomorphisms.} In real applications, finite data cannot cover the whole space $\mathbb{R}^d$. Besides, existing approximation theories only guarantee the capability over some bounded compact set $K$. Due to this two facts, it is more appropriate to consider compactly supported functions or diffeomorphisms. A function $f: \mathbb{R}^d \to \mathbb{R}$ is said to be compactly supported, if there exists a compact set $K$ such that $f(\bm{x}) = 0, \text{ for all } \bm{x} \notin K$. Similarly, a diffeomorphism $f$ is said to be compactly supported, if there exists a compact set $K$ such that $f(\bm{x}) = \bm{x}, \text{ for all } \bm{x} \notin K$, which can result in that all the components of $Df$ are compactly supported. We define
\begin{equation*}
    \text{Diff}^k_c(\mathbb{R}^d) \triangleq \{ f \in \text{Diff}^k(\mathbb{R}^d): f\text{ compactly supported} \},  
\end{equation*}
it is easy to see that $\text{Diff}^k_c(\mathbb{R}^d)$ is a subgroup of $\text{Diff}^k(\mathbb{R}^d)$.

Next, we introduce the parametric diffeomorphism.

{\bf $C^k$-parametric diffeomorphisms.} 
A parametric diffeomorphism is a family of diffeomorphisms with some parameter $\alpha$:  $\{ f_\alpha \}_{\alpha \in A}$, where for any given $\alpha$, $f_\alpha$ is a $\mathbb{R}^d \to \mathbb{R}^d$ diffeomorphism. Usually the parameter is described by some vector $\bm{y} \in \mathbb{R}^m$, and $f_{\bm{y}}$ varies continuously or smoothly w.r.t.~$\bm{y}$. We denote $f(\bm{y},\bm{x}) = f_{\bm{y}} (\bm{x})$. For such a  parametric diffeomorphism, we can embed it into a higher dimensional diffeomorphism  $F(\bm{y},\bm{x}) = \big( \bm{y}, f(\bm{y},\bm{x})\big)$. One can verify directly that $F \in \text{Diff}^k(\mathbb{R}^{m+d})$ given that (1) $f({\bm{y}_0, \bm{x}}) \in \text{Diff}^k(\mathbb{R}^d)$ for any fixed ${\bm{y}_0}$; (2) $f$ is $k$-th differentiable w.r.t. $(\bm{y},\bm{x})$.
It is obvious that
\begin{align}\text{Diff}_c^{k,m,d} \triangleq &\{ F \in \text{Diff}_c^{k}(\mathbb{R}^{m+d}) : F(\bm{y},\bm{x}) = (\bm{y}, f(\bm{y},\bm{x})) ~\text{ with } \bm{y}\in \mathbb{R}^{m}, \bm{x}\in\mathbb{R}^d\},  \label{paradiff}  \end{align}
whose elements keep the first $m$ coordinates and change the last $d$ ones, is a subgroup of $\text{Diff}_c^{k}(\mathbb{R}^{m+d})$. 


\subsection{INNs based on parametric coupling flows}

Here we introduce the classical INNs and investigate the (parametric) diffeomorphism space generated by them. 

{\bf Invertible linear transforms.} First, let us define an elementary group of diffeomorphisms - the invertible linear transforms (ILT):
$$\text{ILT}_{d} \triangleq \{\mathcal{L}:  \mathcal{L} \bm{x}^T = A\bm{x}^T + {\bm{b}}^T, A \in \text{GL}_d(\mathbb{R}), {\bm{b}} \in \mathbb{R}^d \}$$
and the parametric case where $\bm{y}$ is included as parameters:
$$
\text{ILT}_{m,d}\triangleq  \{\mathcal{L} : \mathcal{L}\left(\begin{matrix} {\bm{y}^T}\\{\bm{x}^T} \end{matrix}\right) = \left(\begin{matrix} I_m &0\\B&A \end{matrix}\right)\left(\begin{matrix} {\bm{y}^T}\\ 
{\bm{x}^T} \end{matrix}\right)+ \left(\begin{matrix} {\bm{0}}\\{\bm{b}^T} \end{matrix}\right), 
A \in \text{GL}_{d}(\mathbb{R}), {\bm{b}}\in \mathbb{R}^d\},
$$
where $\text{GL}_d(\mathbb{R})$ denotes the set of all regular matrices on $\mathbb{R}^d$.
These groups could be powerful when combined with nonlinear transforms.

{\bf Coupling flows.} We further define invertible coupling flows which are some specific nonlinear transforms:
\begin{align*}\Phi_{d,i,\phi}: 
\mathbb{R}^{d}&\longrightarrow \mathbb{R}^{d}\\
(\bm{x}_{\le i} ,\bm{x}_{> i})
&\longmapsto
\left(\bm{x}_{\le i} ,\phi(\bm{x}_{\le i},\bm{x}_{>i}) \right),
\end{align*}
where $\phi(\bm{x}_{\le i},\cdot) : \mathbb{R}^{d-i} \to \mathbb{R}^{d-i}$ is a diffeomorphism for each fixed $\bm{x}_{\le i}$. Specifically, when $\phi(\bm{x}_{\le i},\bm{x}_{>i}) = \bm{x}_{>i} \odot \exp\left(\sigma(\bm{x}_{\le i})\right) + t(\bm{x}_{\le i})$, it is the so-called affine coupling flow and we denote  $\Phi_{d,i,\sigma,t} =\Phi_{d,i,\phi}$ for such $\phi$. $\sigma, t$  are some functions with $d-i$ output units,  typically modeled with deep neural networks. $\odot$ represents the point-wise product. $\bm{x}_{\le i} = (x_1, \cdots, x_i),$ $\bm{x}_{> i} = (x_{i+1}, \cdots, x_d)$ for $\bm{x} = (x_1,\cdots,x_d)$.
Similarly, for parametric cases, we have:
\begin{align}\Phi_{d,i,m,\phi}: 
\mathbb{R}^{m+d}&\longrightarrow \mathbb{R}^{m+d}\nonumber\\
(\bm{y} ,\bm{x}_{\le i} ,\bm{x}_{> i})
&\longmapsto
\left(\bm{y} ,\bm{x}_{\le i} ,\phi(\bm{y},\bm{x}_{\le i},\bm{x}_{>i})\right), \label{couplinglayer}
\end{align}
and specifically,
 $\phi(\bm{y},\bm{x}_{\le i},\bm{x}_{>i}) = \bm{x}_{>i} \odot \exp\left(\sigma(\bm{y},\bm{x}_{\le i})\right) + t(\bm{y}, \bm{x}_{\le i})$ for affine-type coupling flows and we denote  $\Phi_{d,i,m,\sigma,t} =\Phi_{d,i,m,\phi}$ for such $\phi$.

{\bf Single-coordinate affine coupling flows.} Specifically, we can define a flow with only the last coordinate changed. We call it a single-coordinate affine coupling flow (SACF): Let $\mathcal{H}_d$ be a set of functions from $\mathbb{R}^d$ to $\mathbb{R}$. We define 

$\mathcal{H}$-SACF$_{d}\triangleq \{ \Phi_{d,d-1,\sigma,t}:\sigma, t \in \mathcal{H}_{d-1}\}$, and

$\mathcal{H}$-SACF$_{m,d}\triangleq\{ \Phi_{d,d-1,m,\sigma,t}:\sigma, t \in \mathcal{H}_{m+d-1}\}$.

Note that any multi-coordinates affine coupling flows can be represented by finite composition of single-coordinate affine coupling flows, and thus it suffices to just consider the universality of single-coordinate affine coupling flows.

{\bf Invertible neural networks.} Now let us combine linear invertible transform layers and some coupling flow layers to construct our INNs. Let $\mathcal{G}$ be a set consisting of invertible coupling flows. We define the set of INNs based on $\mathcal{G}$ as 
$$
\mathcal{G}\text{-INN}_{d} \triangleq\{  g_s \circ W_s  \circ \cdots \circ g_1 \circ W_1 : s\in \mathbb{N}, g_i\in \mathcal{G},\\  W_i \in \text{ILT}_{d}\},$$
and similarly the parametric case: 
$$
\mathcal{G}\text{-INN}_{m,d} \triangleq \{  g_s \circ W_s \circ \cdots  \circ g_1 \circ W_1: s\in \mathbb{N}, g_i\in \mathcal{G},\\  W_i \in \text{ILT}_{m,d}\}.
$$
When $\mathcal{G}$ contains $\mathcal{H}\text{-SACF}_{d}$ (or $\mathcal{H}\text{-SACF}_{m,d}$ in parametric cases), it is equivalent to replace $\text{ILT}_{d}$ (or $\text{ILT}_{m,d}$) by the symmetric group $S_d$ containing all the permutation over $d$ coordinates operating on $\bm{x}$. This type of networks are well known as Real-NVPs \cite{Dinh2016Real}.

One can also define other types of coupling flows. Nevertheless, the theoretical guarantee of any coupling flow can be verified by simply checking their universality to single-coordinate transforms, which is the main result of our paper as stated in Thm.~\ref{singleapprox}.

\subsection{Different types of universality and their relations}

In this section, we give a definition of universality. For the definition of different  functional norms, please refer to Appendix~\ref{normdef}. 


\begin{defi}($L^p/ L^\infty/C^k$-universality).
Let $\mathcal{M}$ be a set of measurable mappings from $\mathbb{R}^n$ to $\mathbb{R}^d$. Let $p \in [1,\infty)$, $k \in \mathbb{N}^+$ and let $\mathcal{F}$ be a set of measurable mappings $f: U_f \to \mathbb{R}^d$, where $U_f$ is a measurable subset of $\mathbb{R}^n$ which may depend on $f$. We say that $\mathcal{M}$ has $L^p$  $(\text{or }L^\infty, C^k)$-universality  for $\mathcal{F}$, if for any $f \in \mathcal{F}$ any $\epsilon > 0$, and any compact subset $K \subseteq U_f$, there exists a $g \in \mathcal{M}$ such that, $\Vert  f - g \Vert _{L^p(K)} < \epsilon$  $(\text{or }\Vert  f - g \Vert _{L^\infty(K)} , \Vert  f - g \Vert _{C^k(K)} < \epsilon)$.
\end{defi}

Here we also define distributional universality. 

\begin{defi}(Distributional universality). 
Let $\mathcal{M}$ be a set  a set of measurable mappings from $\mathbb{R}^n$ to $\mathbb{R}^d$. We say that $\mathcal{M}$ is a distributional universal approximator if for any absolutely continuous probability measure $\mu$ over $ \mathbb{R}^n$ w.r.t. Lebesgue measure, and any probability measure $\nu$ over $\mathbb{R}^d$, there exists a sequence $\{g_i\}_{i=1}^\infty \subseteq \mathcal{M}$ such that $(g_i)_\ast\mu$ converges to $\nu$ in distribution as $i \to \infty$, where $(g_i)_\ast\mu (A) \triangleq \mu\left( g_i^{-1}(A)\right)$ for any measurable $A$.
\end{defi}

If $\mathcal{M}$ has the $C^k$-universality, $\mathcal{M}$ has the $L^\infty$-universality because $C^k(K)$ is dense in $L^\infty (K)$ for any compact $K$. Also, if $\mathcal{M}$ has the $L^\infty$-universality, $\mathcal{M}$ has the $L^p$-universality for any $1 \le p < \infty$. If an $L^p$-universality is satisfied for some $1 < p < \infty$, then for any $1 \le q < p$, $L^q$-universality is ensured. Finally, $L^1$-universality implies distributional universality but the distributional universality does not imply $L^1$-universality. 

\subsection{Related works for universality of flow models}
 According to the model type, we divide the literature into three categories as follows.

{\bf Coupling (or triangular) flows.}
\cite{Hyvrinen1999NonlinearIC,Bogachev2004TriangularTO} proved that the distributional universality of a flow family $h$ can be deduced, if $h$ is dense in the set of all monotone functions by pointwise convergence topology. In particular, \cite{Huang2018NeuralAF} proved the distributional universality for neural autoregressive flows, and  \cite{Jaini2019SumofSquaresPF} proved the distributional universality for sum-of-square flows. \cite{Teshima2020CouplingbasedIN} generalized these results, proving that affine coupling flows have $L^p$ universality for $1 \le p < \infty$, and autoregsive flows and sum-of-square flows have $L^\infty$ universality. However, the approximability to derivatives remains untouched.

{\bf Non-triangular flows.}
The expressiveness of general non-coupling flows is not well studied, as most existing works restrict the form of nonlinearity to certain types, e.g., planar and radial flows \cite{JimenezRezende2015VariationalIW}, Sylvester flows \cite{Berg2018SylvesterNF}, in order to easily compute the determinant of the Jacobian matirx and the inverse maps. \cite{Kong2020TheEP} gave a first study
over some specific distributions, while the universal distributional results still remains unknown. Another type of non-coupling flows, iResNet \cite{Behrmann2019InvertibleRN}, was proposed based on residual network (ResNet) \cite{He2016DeepRL} to improve nonlinearity as well as the computation efficiency of the approximated log determinant.  \cite{Zhang2019ApproximationCO} proved that iResNet, capped by a linear layer or with extra dimensions, has $C^0$-universality.

{\bf Continuous time flows.}
The ODE-based method is also a major class of flow models as introduced in \cite{Chen2018ContinuousTimeFF,Grathwohl2019FFJORDFC,Chen2018NeuralOD,Salman2018DeepDN}.  \cite{Chen2018NeuralOD} gave counterexamples for the $C^0$-universality of neural ODEs, however its distributional universality is not yet addressed.  An ``augmented" neural ODEs was proposed by \cite{Dupont2019AugmentedNO} and then analyzed by \cite{Zhang2019ApproximationCO}. They showed that embedding the original $d$ dimensional neural ODE into a $d+1$ dimensional space can bypass the original counterexamples, achieving a $C^0$-universality.

\section{Main results}
In this section, we present our main results. Section~\ref{nonparametric} provides a proof for non-parametric cases, showing that: 1) the $C^k$-universality over the $\text{Diff}_c^k(\mathbb{R}^d)$ can be achieved if the $C^k$-universality over a simple class of diffeomorphisms (only a single coordinate altered) is ensured; 2) affine coupling layers with one zero-padding has $C^k$-universality over such a simple class of diffeomorphisms, thus over the whole $\text{Diff}_c^k(\mathbb{R}^d)$. In Section~\ref{parametric}, we generalize the result to the parametric case using similar proof framework. 

\subsection{Non-parametric cases} \label{nonparametric}
Here we outline the main steps of our proof, and the complete proof is available in Appendix~\ref{proof}.

Recall the space $\text{Diff}_c^k(\mathbb{R}^d)$: All the $C^k, \mathbb{R}^d \to \mathbb{R}^d$ diffeomorphisms which are compactly supported. There are two advantages of choosing $\text{Diff}_c^k(\mathbb{R}^d)$ instead of $\text{Diff}^k(\mathbb{R}^d)$ to prove the universality: firstly the compactly supported property greatly simplifies the structure of this group; secondly for any diffeomorphism $F \in \text{Diff}^k(\mathbb{R}^d)$ and any compact set $K$, there exists a compactly supported diffeomorphism $f \in \text{Diff}_c^k(\mathbb{R}^d)$ such that $F|_K =f |_K$, i.e., $F(\bm{x}) = f(\bm{x}), \text{ for all } \bm{x} \in K$. Thus, to prove the universality for $\text{Diff}^k(\mathbb{R}^d)$, we only need to consider $\text{Diff}_c^k(\mathbb{R}^d)$.

Directly proving the universality to a general diffeomorphism is difficult. It is beneficial to decompose the original hard problem into a series of simpler ones as follows.
\begin{defi}(Single-coordinate transforms.) A diffeomorphism $\tau \in \text{Diff}_c^k(\mathbb{R}^d)$ is called single-coordinate transform, if $\tau(\bm{x}) = \left(x_1,x_2,\cdots,x_{i-1},\tau_i(\bm{x}),x_{i+1},\cdots,x_d\right)$,  i.e., only one coordinate is altered. Here $\tau_i(\bm{x}) = \tau_i(x_1,\cdots,x_d)$ is a function from $\mathbb{R}^d$ to $\mathbb{R}$ which is monotonic for $x_i$ . We denote the set of all these single-coordinate transforms as $S_c^{k,d}$.
\end{defi}
One may wonder how many diffeomorphisms can be constructed using a composition of these simple single-coordinate transforms. The answer is surprising: the minimum subgroup containing $S_c^{k,d}$ in $\text{Diff}_c^k(\mathbb{R}^d)$ is $\text{Diff}_c^k(\mathbb{R}^d)$ itself. That is to say, all the diffeomorphisms in $\text{Diff}_c^k(\mathbb{R}^d)$ can be constructed from $S_c^{k,d}$.
\begin{theorem} \label{decomposition1}
$S_c^{k,d}$ is a generator of $\text{Diff}_c^k(\mathbb{R}^d)$: for any subgroup
$H \subseteq \text{Diff}_c^k(\mathbb{R}^d)$ s.t. $S_c^{k,d}\subseteq H$, $H = \text{Diff}_c^k(\mathbb{R}^d)$. That is, for any $f \in \text{Diff}_c^k(\mathbb{R}^d)$, there exists $\tau_1, \tau_2, \cdots, \tau_s \in S_c^{k,d}, s\in \mathbb{N}$ s.t. $f = \tau_s \circ \tau_{s-1}\circ \cdots \circ \tau_1$.
\end{theorem}
\begin{proof}
The following two steps are sketched:

1) For any subgroup $H \subseteq \text{Diff}_c^k(\mathbb{R}^d)$ such that $S_c^{k,d}\subseteq H$, $H$ contains all the near-identity diffeomorphisms (Def.~\ref{near-id}). The detail is stated in Cor.~\ref{totaldecomposition}.

2)  For any subgroup $H \subseteq \text{Diff}_c^k(\mathbb{R}^d)$ such that $H$ contains all the near-identity diffeomorphisms, $H =\text{Diff}_c^k(\mathbb{R}^d)$ (Lem.~\ref{neariddecomposition}) .
\end{proof}
\begin{remark}
An upper bound on $s$ is given in the scope of our proof. Given $f\in \text{Diff}_c^k(\mathbb{R}^d)$, we denote $\text{dist}_{\text{Diff}_c^1(\mathbb{R}^d)}(f, I) = \ell$ for some $\ell> 0$, the minimal length of paths between $f$ and the identity map $I$ lying in $\text{Diff}_c^1(\mathbb{R}^d)$ with the $C^1$-norm induced metric. By Cor.~\ref{totaldecomposition}, $f$ can be decomposed into $s_1$ many $(\frac{1}{d-1},1)$-near-identity diffeomorphisms with $s_1\approx (d-1)\ell$; by Lem.~\ref{neariddecomposition}, each  $(\frac{1}{d-1},1)$-near-identity diffeomorphism can be decomposed into at most $d$ single-coordinate transforms, thus $s \le s_1 \cdot d \approx d(d-1)\ell$. The evaluation for $\ell$ is difficult, and thus we only give a lower bound here: $\ell \ge \Vert f - I\Vert_{C^1}$. However, if $f$ is not far from the identity, the lower bound is a good estimator. Please refer to \cite{Banyaga1997The} for more details.
\end{remark}
Thm.~\ref{decomposition1} decomposes the original general complex diffeomorphism into finitely-many simple diffeomorphisms. It is natural to ask: if we can approximate all the $S^{k,d}_c$ well, can we also approximate $\text{Diff}_c^k(\mathbb{R}^d)$ well? Thm.~\ref{compos-approx} gives a positive answer to the question.
\begin{theorem} \label{compos-approx}
(Approximation for composition.)
Suppose $G$ is a group of diffeomorphisms over $\mathbb{R}^d$. Given a set of diffeomorphisms $F$ over $\mathbb{R}^d$, denote $\mathcal{F}$ as the semigroup generated by $F$. If $G$ has $C^k$-universality for $F$,   $G$ has $C^k$-universality for $\mathcal{F}$. If $F$ is inverse-invariant (i.e., for any $f \in F, f^{-1} \in F$), then $\mathcal{F}$ is a group.
\end{theorem}
\begin{proof}
We only need to prove the following statement: If $G$ has $C^k$-universality for $f_1, f_2 \in\mathcal{F}$, then $G$ has $C^k$-universality for $f_1\circ f_2 \in\mathcal{F} $.
By the property of semigroups, we know that $G$ has $C^k$-universality for the whole $\mathcal{F}$. Details can be found in~\ref{proof-compos}.
\end{proof}
Combining Thm.~\ref{decomposition1} and~\ref{compos-approx}, noticing that $S_c^{k,d}$ is inverse-invariant (for any $\tau \in S_c^{k,d}, \tau^{-1} \in S_c^{k,d}$), we can immediately obtain following theorem.
\begin{theorem}\label{singleapprox} 
Suppose $G$ is a group of diffeomorphisms over $\mathbb{R}^d$, if $G$ has $C^k$-universality for $S_c^{k,d}$, then $G$ has $C^k$-universality for $\text{Diff}_c^k(\mathbb{R}^d)$. 
\end{theorem}
According to Thm.~\ref{singleapprox}, any INNs that approximate $S_c^{k,d}$ will eventually approximate $\text{Diff}_c^k(\mathbb{R}^d)$. This is a generic result and is first stated in \cite{Teshima2020CouplingbasedIN} for $L^p$ and $L^\infty$ cases. We generalize this result to $C^k$ cases and also to parametric cases in the next section.

{\bf Affine-Coupling Flows.}
Now we only need to prove the universality of $\mathcal{G}\text{-INN}_{d}$ over $S_c^{k,d}$, when $\mathcal{G}$ contains  $\mathcal{H}\text{-SACF}_{d}$. However, even for such a simple class of diffeomorphisms, $\mathcal{G}\text{-INN}_{d}$ is not able to approximate it very well: existing works \cite{Teshima2020CouplingbasedIN} shows that $\mathcal{G}\text{-INN}_{d}$ has $L^p$ universality over $S_c^{k,d}$. While they gave a negative conjecture for $L^\infty$ universality, let alone the $C^k$-universality.

However, despite the lack of easily computed log-determinate for density evaluation, when the single-transform is embedded into a higher dimension Euclidean space by padding zeros, the approximation ability becomes much stronger. We prove that the approximation can achieve $C^k$-universality, far beyond $L^p$ universality. We first define the canonical immersion operator $\iota_d$, and its pseudoinverse, canonical submersion $\pi_d$:
\begin{align*}
    \iota_d &\colon 
    (x_1,\cdots,x_d) \mapsto (x_1\cdots,x_d,0),\\
    \pi_d &\colon 
    (x_1,\cdots,x_d,x_{d+1}) \mapsto (x_1\cdots,x_d).
\end{align*}
In Thm.~\ref{single-approx}, we show that by raising the dimension by 1, we can achieve $C^{k}$ -universality. A proof is stated in~\ref {single-approx-chapter}.
\begin{theorem} \label{single-approx}
For any compact set $K\subseteq\mathbb{R}^d$, and any $\tau \in S_c^{k,d}$, $\tau: (x_1, x_2, \cdots, x_{d-1}, x_{d}) \longrightarrow \left(x_1, x_2, \cdots, x_{d-1}, \tau_d(\bm{x})\right) $, there exists $\tilde \tau$ in $ \mathcal{G}\text{-INN}_{d+1}$ with $\mathcal{G} \supseteq  \mathcal{H}\text{-SACF}_{d+1}$, such that $\Vert \tau - \pi_d\circ \tilde \tau\circ \iota_d\Vert _{C^{ k}(K)} < \epsilon$, given that $\mathcal{H}_d$ has $C^k$-universality over functions $K\to \mathbb{R}$. 
\end{theorem} 
We show in Appendix~\ref{eliminating} that, even with this zero-padding, there still exist efficient ways to compute the inverse map of the neural network with a dimension eliminating operator.





\subsection{Parametric cases}\label{parametric}

Following Section~\ref{nonparametric}, we generalize the results to parametric cases,   with some changes to cope with parameters.

Recall the definition of compactly supported parametric diffeomorphisms group Eq.~\eqref{paradiff}, where $\bm{y}$ acts as parameters. One might raise a question: why not directly use the result in non-parametric case, as all the elements in  $\text{Diff}_c^{k,m,d}$ are also in $\text{Diff}_c^{k}(\mathbb{R}^{m+d})$? Note that the decomposition theorem only ensures that $S_c^{k,m+d}$ can generate $\text{Diff}_c^{k}(\mathbb{R}^{m+d})$ and then $\text{Diff}_c^{k,m,d}$, it does not guarantee that  $S_c^{k,m+d}\cap \text{Diff}_c^{k,m,d}$, which is the set of transforms acting on a single coordinate of $\bm{x}$, can generate $\text{Diff}_c^{k,m,d}$. Without such a guarantee, we have to use operators that may alter coordinates of both $\bm{y}$ and $\bm{x}$, and the coordinates of $\bm{y}$ in the output layer should be exactly the same as the coordinates of $\bm{y}$ in the input layer, which restricts the flexibility and increases the  approximation difficulty. Thus we derive new dedicated theorems for parametric cases.

We denote $S_c^{k,m,d}\triangleq S_c^{k,m+d}\cap \text{Diff}_c^{k,m,d}$.

\begin{theorem} \label{decompositionpara}
$S_c^{k,m,d}$ is a generator of $\text{Diff}_c^{k,m,d}$: for any subgroup
$H \subseteq \text{Diff}_c^{k,m,d}$ s.t. $S_c^{k,m,d}\subseteq H$, $H = \text{Diff}_c^{k,m,d}$. That is, for any $f \in \text{Diff}_c^{k,m,d}$, there exists $\tau_1, \tau_2, \cdots, \tau_s \in S_c^{k,m,d}, s\in \mathbb{N}$ s.t. $f = \tau_s \circ \tau_{s-1}\circ \cdots \circ \tau_1$.
\end{theorem}
\begin{proof}
The proof follows Thm.~\ref{decomposition1}:

     1) For any subgroup $H \subseteq \text{Diff}_c^{k,m,d}$ such that $S_c^{k,m,d}\subseteq H$, $H$ contains all the  near-identity diffeomorphisms belonging to  $\text{Diff}_c^{k,m,d}$ . The proof follows Cor.~\ref{totaldecomposition}. 
    
     2) For any subgroup $H \subseteq \text{Diff}_c^{k,m,d}$ such that $H$ contains all the near-identity diffeomorphisms belonging to $\text{Diff}_c^{k,m,d}$, $H =\text{Diff}_c^{k,m,d}$ (Thm.~\ref{neariddecompositionpara}) .
\end{proof}
With the assistance of the above theorem, we can easily obtain the $C^k$-universality for $\text{Diff}_c^{k,m,d}$ by simply following the procedure in section~\ref{nonparametric}.
\begin{theorem}\label{singleapproxpara}
Suppose $G$ is a group of diffeomorphisms over $\mathbb{R}^{m+d}$. If $G$ has $C^k$-universality for $S_c^{k,m,d}$, then $G$ has $C^k$-universality for $\text{Diff}_c^{k,m,d}$. 
\end{theorem}
A proof of   $\mathcal{G}\text{-INN}_{m,d}$ universality over $S_c^{k,m,d}$, when $\mathcal{G}$ contains  $\mathcal{H}\text{-SACF}_{m,d}$, follows Thm.~\ref{single-approx} in section~\ref{nonparametric}, and thus the detail of the proof is skipped here.

\begin{figure*}[ht]
\vspace*{-5mm}
    \vskip 0.2in
    \centering
    \includegraphics[width = 0.9 \linewidth]{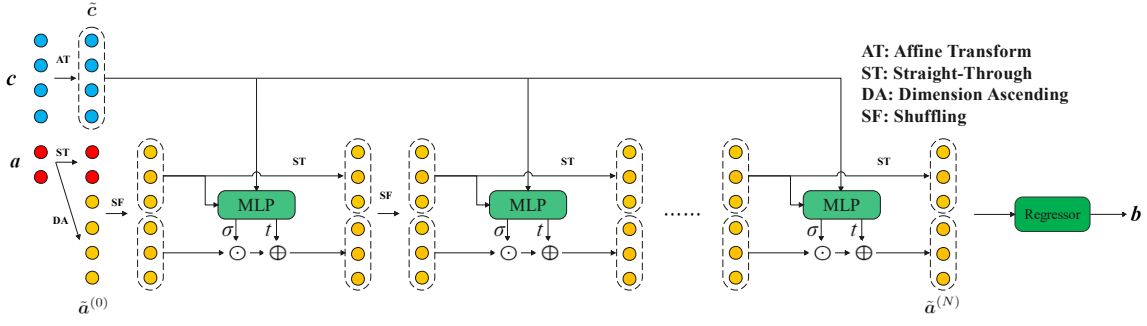}
    \caption{Network structure of affine coupling flow for contextual BO}
\vspace*{-2mm}
    \label{Fig: para-flow-architecture}
\end{figure*}
\section{\texttt{Para-CFlows} for Contextual Bayesian Optimization}


In real applications such as control and optimization of telecommunication networks, power grid systems, and electronic systems, we aim to optimize a complicated system under various situations. We denote by $\Phi(\bm{c}, \bm{a})$ as the system of interest, where $\bm{a}$ is a set of actions one can tune and $\bm{c}$ is a context vector representing exogenous stimuli. 
We assume $\bm{a} \in \mathcal{A} \subseteq\mathbb{R}^{d_a}$, $\bm{c} \in \mathcal{C}\subseteq\mathbb{R}^{d_c}$ where $\mathcal{A}$ and $\mathcal{C}$ are  connected and have positive Lebesgue measure. 
\begin{assumption}\label{info-preserve-assumption}
The system inner state $\Phi(\bm{c}, \bm{a}): \mathcal{C}\times\mathcal{A}  \to \mathbb{R}^{d_\phi}$ is a differentiable function and preserves full information of the actions, i.e., $\text{rank}\left(\frac{\partial \Phi}{\partial \bm{a}}\right) = d_a$.
\end{assumption}


In most cases, we are only interested in optimizing a function $b: \mathbb{R}^{d_\phi} \to \mathbb{R}$ of the inner state $\Phi(\bm{c},\bm{a})$, i.e.,
\[
\bm{a}^{\ast} = \arg\max_{\bm{a}} b(\Phi(\bm{c},\bm{a})),
\]
given a context $\bm{c}$. Though we have no exact form of the function, we can query its value under specific action configuration at a given context. Our task is to quickly find the optima at each context using as few queries as possible. 

Many works have been proposed using contextual Gaussian process \cite{Krause2011ContextualGP,Chung2020OfflineCB}, as well as 
deep neural networks \cite{Allesiardo2014ANN,Zhou2020NeuralCB}, random forest \cite{Fraud2016RandomFF,Hutter2011SequentialMO}, etc. to construct the surrogate and acquisition function guiding the next query. However, existing methods only work well for action-\textbf{sensitive} tasks, i.e. the norm of the partial derivatives of $b$ w.r.t.~ $\bm{a}$ is not significantly less than that w.r.t.~$\bm{c}$.
However, in many real applications, it is more common that action is \textbf{insensitive}. A naive implementation concatenating $\bm{a}$ and $\bm{c}$ together as inputs of a model usually does not work well, as the surrogate model $\mu(\bm{c},\bm{a})$ can easily degenerate to $\mu(\bm{c})$ due to limited sample size. The problem is even worse when $dim(\bm{c})\gg dim(\bm{a})$. In most cases such a degenerated model performs quite bad.

To address the challenge, we propose our \texttt{Para-CFlow} (see Fig.\ref{Fig: para-flow-architecture} for the concrete model architecture) to learn a surrogate that preserves the correct sensitivity of the actions.

We first map $\bm{a}$ into a higher dimensional space $\mathbb{R}^d$ ($d > d_a$):
\[
    \phi_0: 
    \bm{a} \mapsto  (  \bm{a},\bm{a}\bm{W})
    \triangleq\tilde{\bm{a}}^{(0)}, 
\]
where $\bm{W}\in \mathbb{R}^{d_a \times (d-d_a)}$. By construction, $\text{rank}(\phi_0) = d_a$.

We further compose $\phi_0$ with a series of invertible affine coupling layers $\phi_i, 1\le i \le N$ operating in $\mathbb{R}^d$ (Eq.~\eqref{couplinglayer}). For any $\bm{x}\in \mathbb{R}^d$, let $d' = \max\{d_a, [\frac{d}{2}]\}$, we define 
\begin{align}\phi_{\bm{c},i}(\bm{x}) &= \phi_{\bm{c},i}(\bm{x}_{\le d'},\bm{x}_{>d'}) \nonumber 
= \left(\bm{x}_{\le d'}, \bm{x}_{>d'} \odot \exp(\sigma_i(\bm{c},\bm{x}_{\le d'})) + t_i(\bm{c}, \bm{x}_{\le d'})\right),\nonumber \end{align}
$\sigma_i, t_i$  are some functions with $d-d'$ output units, implemented by any nonlinear functions like deep neural networks and we parameterize them by $\bm{\Theta}_{\sigma,i}$ and $\bm{\Theta}_{t,i}$.

Now we can define $\tilde{\bm{a}}^{(i)}$  following the recurrence relation: $\tilde{\bm{a}}^{(i)} = \phi_i{P}_i(\tilde{\bm{a}}^{(i-1)}), 1\le i \le N$, where ${P}_{i}$ are any fixed random permutes. Thus, we write 
\begin{equation}\label{phi}
     \tilde{\bm{a}}^{(N)} =  (\phi_{\bm{c},N}{P}_N)\circ\cdots\circ(\phi_{\bm{c},1}{P}_{1})\circ\phi_0(\bm{a}) \triangleq\phi(\bm{c}, \bm{a}),
\end{equation}
which is followed by a simple function $\hat b: \mathbb{R}^{d_a} \to \mathbb{R}$ represented by a simple neural network parametrized by $\bm{\Theta}_{\hat b}$. Then the whole model is trained by minimizing the Mean Square Error (MSE):
\[
    \min_{\bm{W}, \bm{\Theta}} 
    \frac{1}{M} \sum_{1\le j \le M} \left(b_j - \hat{b}(\phi(\bm{c}_j, \bm{a}_j))\right)^2,
\]
where $\bm{\Theta} = \{\bm{\Theta}_{\sigma,1}, \dots, \bm{\Theta}_{\sigma,N}, \bm{\Theta}_{t,1}, \dots, \bm{\Theta}_{t,N}, \bm{\Theta}_{b}\}$ 
and $b_j$ are sampled data of the system given $(\bm{c}_j,\bm{a}_j)$ which we assume to have the form $b_j = b( \Phi(\bm{c}_j,\bm{a}_j))$. 

As such, we construct a neural network that preserves full information (the rank) of $\bm{a}$. We further investigate if \texttt{Para-CFlow} is a universal approximator to all these information-preserving systems satisfying Asmp.~\ref{info-preserve-assumption}.

\begin{theorem}
For any system $\Phi: \mathcal{C}\times\mathcal{A}  \to \mathbb{R}^{d_\phi}$ satisfying  Asmp.~\ref{info-preserve-assumption}, if $d \ge 2d_\phi$ and $\Phi$ have up to $k$-th order derivatives, then for any $\epsilon > 0$, we can always find a $\phi$  in the form~\eqref{phi}, such that $\Vert  \Phi(\bm{c},\bm{a}) - \pi \phi(\bm{c},\bm{a})\Vert_{C^k(\mathcal{C}\times\mathcal{A})} \le \epsilon$, given that $\sigma_i, t_i, \hat b$ are $C^k$-universal function approximator. Here $\pi$ is the canonical submersion: $\mathbb{R}^{d} \to \mathbb{R}^{d_\phi}$, $(x_1,x_2,\cdots, x_{d_\phi},\cdots,x_d) \mapsto (x_1,x_2,\cdots,x_{d_\phi})$.
\end{theorem}
\begin{proof}
By  the famous Constant Rank Theorem \cite{Boothby1975AnIT}, there exists a parametric ($\bm{c}$ plays the role of parameter) diffeomorphism $f_{\bm{c}}(\cdot): \mathbb{R}^{d_\phi} \to \mathbb{R}^{d_\phi}$ such that 
$$f_{\bm{c}}(\bm{a}, \bm{0}_{1\times(d_\phi-d_a)}) = \Phi(\bm{c},\bm{a}), \text{ for all } (\bm{c},\bm{a}) \in \mathcal{C}\times\mathcal{A}.$$ Then it is straightforward to approximate $f_{\bm{c}}$ with Thm.~\ref{single-approx},~\ref{singleapproxpara} given that $d' = \max\{d_a, [\frac{d}{2}]\} \ge d_\phi$.
\end{proof}

So far, we have proposed our \texttt{Para-CFlow} model as a surrogate model for contextual BO tasks. 
As we have already proven, \texttt{Para-CFlow} enjoys the universal approximation property and it preserves non-degeneracy of $\bm{a}$, no matter how many layers are stacked to model the interaction between $\bm{ a}$ and $\bm{c}$. \texttt{Para-CFlow} is expected to guide the contextual BO algorithms to quickly find the optimal action  at any different context. Note that in the setting of contextual BO, we can adopt existing state-of-the-arts such as Bayesian deep learning, deep ensembles or Thompson Sampling, to facilitate the uncertainty modeling.

\section{Experiments}
In this section, 
we conduct experiments to verify the expressiveness of our proposed \texttt{Para-CFlow} 
and demonstrate its advantage as neural surrogate models in contextual BO.

\subsection{Learning diffeomorphism ``Taiji''}\label{taiji}
We demonstrate the expressiveness of our model to learn a parametric diffeomorphism with a synthetic task called ``Taiji". The result demonstrates that, not only the $0$-th order information, but also the $1$-st order information (derivatives) can be learned with good accuracy.

We define our target parametric diffeomorphism: 
\begin{align}f_y: (\rho \cos\theta, \rho \sin\theta) \mapsto (\rho \cos \tilde\theta, \rho \sin\tilde\theta),\nonumber\\
\text{where }\tilde \theta = \theta + y \arccos(\min(\rho,1)). \label{taijieq}\end{align}
On the unit circle $\rho = 1$, it is not differentiable, thus we expect the trained model to have irregular derivatives over the unit circle, while maintain accurate derivatives elsewhere.

First, we design a task with sufficient amount of data to verify the $C^1$ approximation ability of \texttt{Para-CFlows}. We generate samples $\{(\bm{x}_i, y_i)\}_{1 \le i \le 30000}$ with
$\bm{x}_i \sim \mathcal{U}[-1, 1]^2$ and $y_i \sim \mathcal{U}[0, 1]$ independently. 
Under the polar coordinate representation, i.e., $\bm{x}_i = (\rho_i \cos \theta_i, \rho_i \sin \theta_i)$, we calculate the target $f_{y_i}(\bm{x}_i) = (\rho_i \cos \tilde \theta_i,\rho_i \sin \tilde \theta_i)$ according to Eq.~\eqref{taijieq}.
We use a 6-layers affine coupling flows, each layer composed with a random permute. The coupling functions $\sigma,t$ are implemented by 1-hidden-layer Multi-Layer Perceptron (\texttt{MLP}) with hidden-unit number comparable to input dimension. We pad a zero to the input $\bm{x}$, and use the first two output dimensions to compute the loss. 

On the left panel of Fig.~\ref{Fig:taijigraph}, we observe good consistency between the trained model $\hat f_{y}(\bm{x})$ and $ f_{y}(\bm{x})$, when $y \in [0,1]$. Note that we only use samples with $y_i \in [0,1]$ for training, and thus the model generalizability to $y >1$ is not guaranteed. However, as the true diffeomorphism satisfies $f_{y_1}\circ f_{y_2} = f_{y_1 + y_2}$, we  iterate the trained model recurrently and expect the final output to achieve the effect
for $y > 1$. The results on the right panel of Fig.~\ref{Fig:taijigraph} justify that the 
cumulative error from composition does not grow fast, and our model can approximate complex diffeomorphisms with sufficient layers. In Fig.~\ref{Fig:derivative} we plot the derivatives of \texttt{Para-CFlow}, compared to the true one. The non-differentiability of the groundtruth on the unit circle results in irregular behaviors of trained model near the unit circle; elsewhere it shows good consistency.

\begin{figure}[ht]
\vspace*{-7mm}
    \vskip 0.2in
    \centering
    \subfigure{
    \includegraphics[width = 0.45\linewidth,height = 0.5\linewidth]{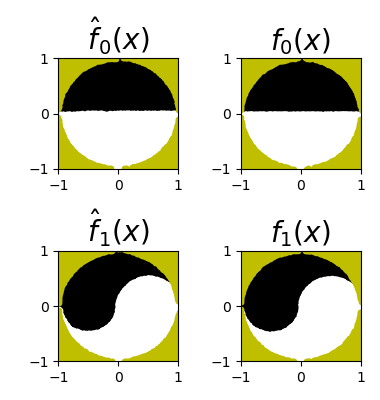}} 
    \subfigure{
    \includegraphics[width = 0.45\linewidth,height = 0.5\linewidth]{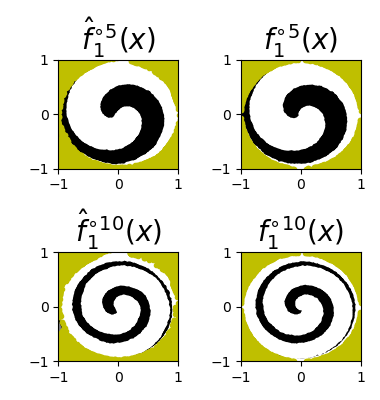}}
 \vspace*{-3mm} 
    \caption{Trained $\hat f_{y}(\bm{x})\text{ v.s. Target }f_{y}(\bm{x})$, learning ``Taiji'' diffeomorphism tasks in section~\ref{taiji}. Black region, $x_2 > 0 \cap \rho \le 1$; white region, $x_2 \le 0 \cap \rho \le 1$; yellow region, $\rho > 1$. $\rho = \sqrt{x_1^2 + x_2^2}$ and $\bm{x} = (x_1,x_2)$. $\hat f_1^{\circ N}$ denotes the $N$-fold composition $\hat f_1\circ\cdots\circ \hat f_1$.}
    \label{Fig:taijigraph}
\end{figure}

\begin{figure}[ht]
\vspace*{-7mm}
    \vskip 0.2in
    \centering
    \subfigure[$\frac{\partial \hat f^{(1)}(y,\bm{x})}{\partial \bm{x}} \text{ v.s. }\frac{\partial f^{(1)}(y,\bm{x})}{\partial \bm{x}}$]{\label{Fig2sub1}
    \includegraphics[height = 0.39\linewidth,width = 0.45\linewidth, trim = 0 0 0 30,clip]{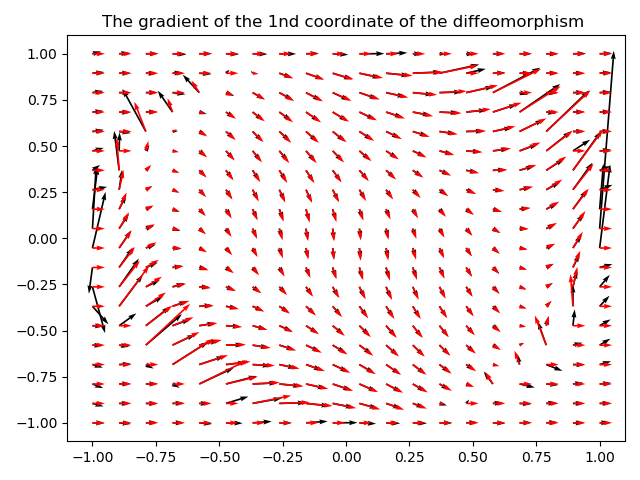}} 
    \subfigure[$\frac{\partial \hat f(y,\bm{x})}{\partial y} \text{ v.s. }\frac{\partial f(y,\bm{x})}{\partial y}$]{\label{Fig2sub3}
    \includegraphics[height = 0.39\linewidth, width = 0.45\linewidth,trim = 0 0 0 30,clip]{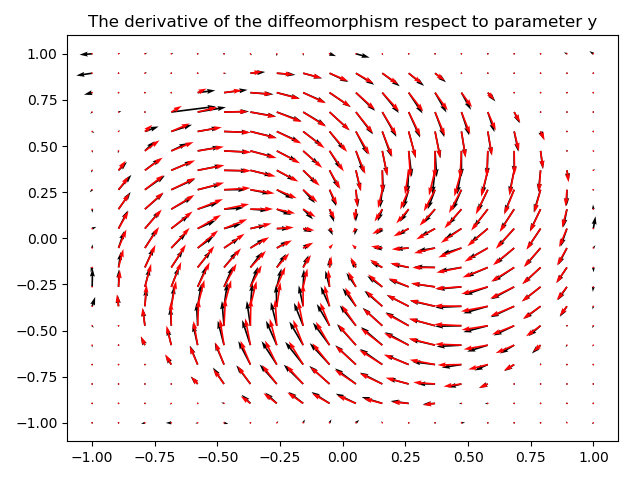}
    }
 \vspace*{-4mm} 
    \caption{Derivatives of the trained model $\hat f_y(\bm{x}) = (\hat f^{(1)}(y,\bm{x}),\hat f^{(2)}(y,\bm{x}))$ with black arrows and the groundtruth $ f_y(\bm{x}) = ( f^{(1)}(y,\bm{x}), f^{(2)}(y,\bm{x}))$ with red arrows, all at $y = 1$. }
   \label{Fig:derivative}
 \vspace*{0mm} 
\end{figure}

\begin{figure}[ht]
\vspace*{-0mm}
    \centering
    \includegraphics[height = 0.54\linewidth ,width = 0.9\linewidth]{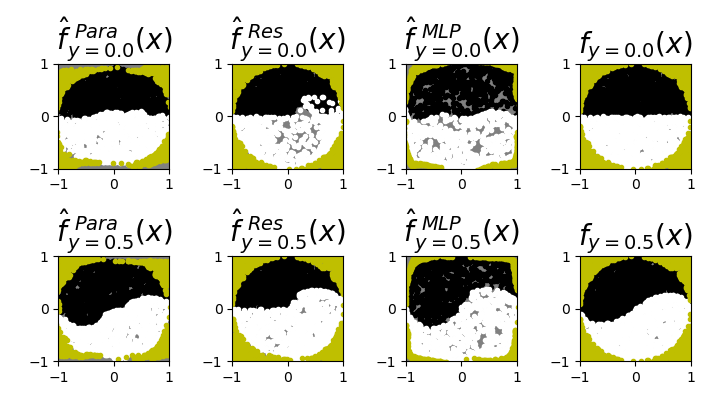}
\vspace*{-2mm}
    \caption{Comparing \texttt{Para-CFlow} with \texttt{Resnet} and \texttt{MLP} with small sample size and parameter $y^{(i)}=y, 1\le i \le 100$. Grey color represents area where no point resides.}
    \label{Fig:compare}
\vspace*{-2mm}
\end{figure}

 Next we design a more challenging task to compare \texttt{Para-CFlow} against a simple \texttt{MLP} with a hidden-layer-setting as $(128,64,32)$, and  Residual Nets (\texttt{Resnet}) \cite{He2016DeepRL}, that concatenates $\bm{y}$ and $\bm{x}$ as the network's input and then appends $\bm{x}$ to the output of each layer forward for prediction. Unlike the previous experiment where a scalar $y$ is used, we generate 3000 samples of a 100-dimensional parameter vector $\bm{y}_i = (y_i^{(1)},\cdots, y_i^{(100)})$ where $y_i^{(j)}  \sim \mathcal{N}(y_i, 0.16)$ and $y_i\sim \mathcal{U}[0,1]$. Then $(\bm{x}_i,\bm{y}_i)$ is fed to each model to learn the target $f_{\text{avg}  (\bm{y}_i)}(\bm{x}_i)$. Fig.~\ref{Fig:compare} shows that \texttt{Para-CFlow} learn a much more continuous mapping than \texttt{MLP} and \texttt{Resnet}. While \texttt{Resnet} learns relatively better the outer part than \texttt{MLP}, similar to $\texttt{MLP}$ it suffers from bad mapping in the inner part, i.e. the grey area inside the unit circle is much larger than \texttt{Para-Flow}, which represents bad continuity of the learnt mapping.

\subsection{Application to contextual BO}\label{Sec:Exp:BO}
The performance of contextual BO depends not only on the exploration-exploitation strategy but also the surrogate model.
Ensemble of neural networks is widely used for surrogate and uncertainty modeling, which is flexible for single or multiple outputs \cite{Balaji2017Simple}.

In this section,
we ensemble \texttt{Para-CFlow}s as a novel surrogate model for BO. 
We set our experiments with three well-known benchmark functions: \texttt{Ackley}, \texttt{Trid} and \texttt{Rastrigin} 
\cite{benchmark_functions-web}, which have different features. 
Note that the original benchmark tasks are \textbf{contextless}, 
here we simply set the first $d_c$ dimensions as the context vector sampled from a multivariate uniform distribution and leave the last dimension as the action to optimize. 
To construct context-dominating tasks where the reward depends much more on the context than the action, we set $d_c\gg 1$. More details are available in Appendix \ref{Sec:Appedix:Exp}.

In all benchmarks above, the contexts $\bm{c}$ are uniformly sampled at random from $[-3,3]^{d_c}$ sequentially while the action $a$ is optimized over $\texttt{linspace}(-3,3,100)$. To compare different surrogate models, we report the cumulative regrets w.r.t. the optimum alongside the context sequence.

The used ensemble-like surrogate models for contextual BO are listed as follows: 
(1) \texttt{Para-CFlow}: Ensemble of \texttt{Para-CFlow}s for prediction with uncertainty; 
(2) \texttt{MLP}: Ensemble of \texttt{MLP}s for uncertainty modelling by concatenating context $\bm{c}$ and $a$ to obtain input $(\bm{c}, a) \in \mathbb{R}^{d_c+1}$; 
(3) \texttt{MLP-Ascend}: Similar to \texttt{MLP}, the difference is that it firstly raises both $\bm{c}$ and $a$ to the same hidden dimension $\max(5, d_c)$ and then concatenates them as inputs to a \texttt{MLP}; 
(4) \texttt{Resnet}: Ensemble of multiple \texttt{Resnet} that concatenates $\bm{c}$ and $a$ as the network's input and then appends $a$ to the output of each layer forward for prediction.

We implement each of the aforementioned neural networks with similar sizes (details can refer to Appendix \ref{Sec:Appendix:Exp:Settings}) using PyTorch \cite{Adam2019PyTorch}.
We ensemble 5 base models as the surrogate model and use the empirical standard deviation of the base model outputs as the uncertainty measurement.
We use Lower Confidence Bound (LCB) with $\kappa = 1$ as trade-off strategy 
(The results of using Thompson's sampling \cite{Russo2018Tutorial} are presented in Appendix \ref{Sec:Appedix:Exp:CumulativeRegrets})
and repeat BO with different surrogates for 10 trails for each experiment setting.
In the first 100 steps,
we draw a same batch of uniform random samples of $a$ to initialize all surrogates.
After that we conduct BO step-by-step and 
update the surrogates per 100 steps using newly collected samples respectively.

We report the mean and standard deviation of the cumulative regrets for different surrogate models in Fig.~\ref{Fig:Exp:CumulativeRegrets:LCB}.
The dimensionality of context is set to 5, 10, 20,
which is significantly greater than that of action,
to demonstrate the capability of \texttt{Para-CFlow} in preserving the sensitivity of action.
We observe a significant improvement in optimization for applying \texttt{Para-CFlow} in neural surrogate modeling, 
which is probably because it preserves all the information of action, 
whereas compared methods do not. 
To further verify this conjecture, 
we train all models with uniformly randomly sampled data (both context and action) of various sizes on \texttt{Trid} 
and then test the trained models on 10,000 testing contexts, 
each with sweeping the action space.
We calculate the Kendall's Tau (KT) score \cite{Kendall1938} between $f_c(a)$ and $\hat{f}_c(a)$,
which measures the action order consistency between the groundtruth and the prediction.
As shown in Fig.~\ref{Fig:KT},
\texttt{Para-CFlow} exhibits higher KT scores than the other models under higher dimensional context,
which demonstrate that it indeed preserves critical information of the action.

\begin{figure}[t]
    \centering
    \includegraphics[width=\linewidth]{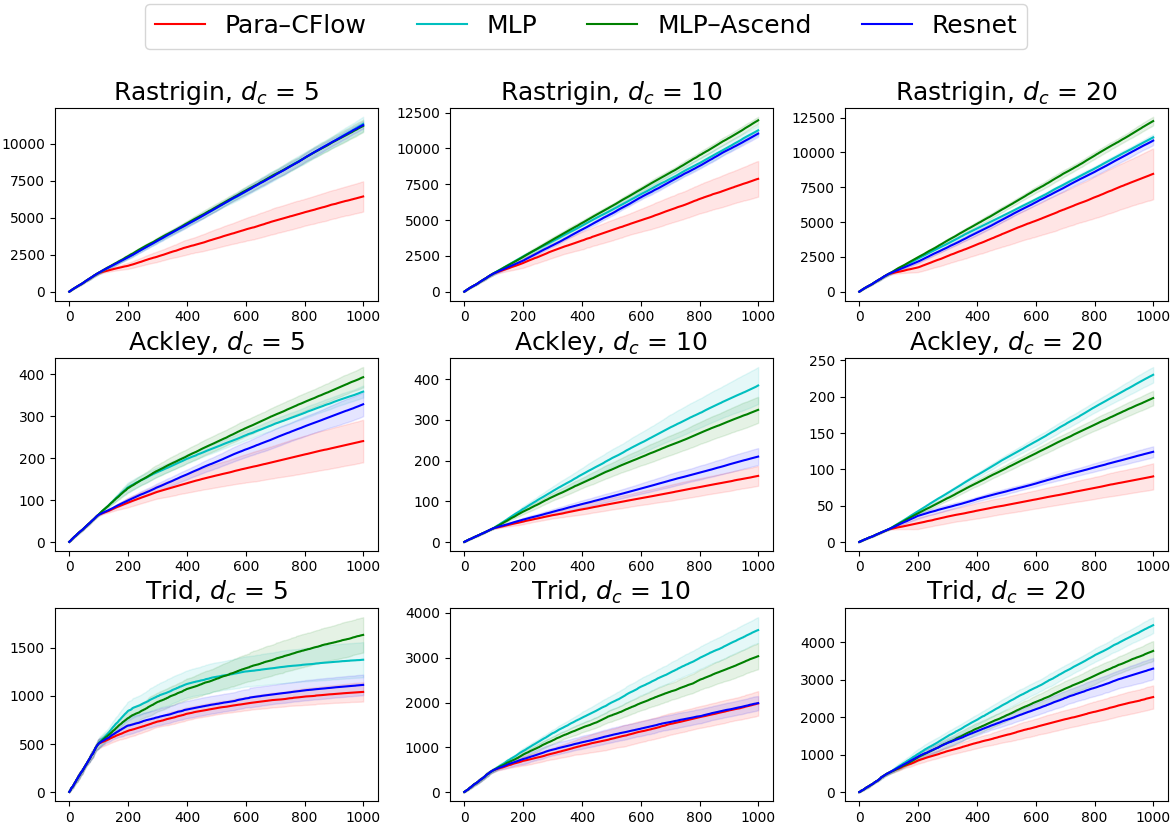}
    \vspace{-2mm}
    \caption{The mean and standard deviation of the cumulative regret under 10 independent trials with the context dimension as 5, 10 and 20 on different benchmarks.}
    \label{Fig:Exp:CumulativeRegrets:LCB}
     \vspace*{-0mm} 
\end{figure}

\begin{figure}[t]
    \centering
    \includegraphics[width=\linewidth]{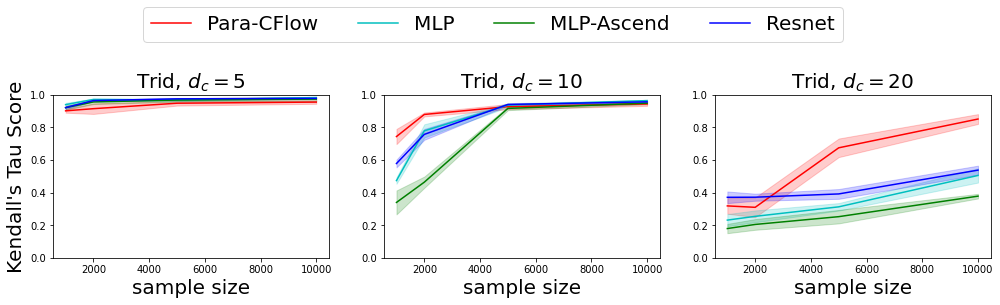}
    \vspace{-4ex}
    \caption{KT scores calculated from 5 independent trials with the context dimension as 5, 10 and 20 on Trid.}
    \label{Fig:KT}
     \vspace*{-2mm} 
\end{figure}

\section{Conclusion}
In this paper,  we firstly prove the equivalence between the $C^k$-universality over compactly-supported diffeomorphisms and that over single-coordinate transforms, resulting in the $C^k$ universal approximation ability of affine coupling flows.
Furthermore, we generalize the main theorems to parametric cases and propose a practical model called \texttt{Para-CFlow}s which could serve as a good surrogate for contextual BO. 
With good capabilities both in universal approximation and robust sensing of critical features in parametric diffeomorphisms, we empirically exhibit the advantages of \texttt{Para-CFlow} using various  benchmarks. 



\clearpage
\bibliography{para-flow}
\bibliographystyle{plain}

\clearpage
\appendix
\section{Different functional norms and corresponding function spaces}\label{normdef}

Here we introduce some preliminary definitions on function norms and functions spaces involved in this paper. 

For a measurable mapping $f: \mathbb{R}^n \to \mathbb{R}^d$ and a subset $K \subseteq \mathbb{R}^n$, we define:
\begin{align*} 
    \Vert  f \Vert _{L^p(K)} &\triangleq \left( \int_{K} \Vert  f(\bm{x}) \Vert ^p d\bm{x}\right)^{1/p}, 1 \le p < \infty, \\
    \Vert  f \Vert _{L^\infty(K)} &\triangleq \lim_{p \to \infty} \Vert  f\Vert _{L^p(K)} \xlongequal{f \text{ cont.}} \sup_{\bm{x} \in K} \Vert  f(\bm{x})\Vert, 
\end{align*}
where $\Vert\cdot\Vert$ can be any norm on $\mathbb{R}^d$, as norms on the finite-dimensional vector space are all equivalent. For simplicity, we choose the maximum norm on $\mathbb{R}^d$, i.e., $\Vert (x_1,x_2,\cdots,x_d)\Vert  = \max_{1\le i\le d} |x_i|$.

We can also consider the norm containing derivative information when $f$ is $k$-th differentiable:
\begin{equation*}
\Vert f\Vert _{C^k(K)} \triangleq \sum_{|\bm \alpha| \le k} \Vert  D^{\bm \alpha} f\Vert _{L^{\infty}(K)},
\end{equation*}
where ${\bm \alpha} = (\alpha_1, \alpha_2,\cdots,\alpha_d) \in \mathbb{N}^d$, $|\bm \alpha| = \sum_{i = 1}^d \alpha_i$ and $D^{\bm \alpha} f = \frac{\partial^{|\bm \alpha|} f}{\partial x_1^{\alpha_1} \cdots \partial x_d^{\alpha_d}}.$

With these norms, we can define corresponding function spaces:
\begin{align*}L^p(K) &\triangleq \{\text{Domain}(f) = K : \Vert  f \Vert _{L^p(K)} < \infty  \} \\
L^\infty(K) &\triangleq \{\text{Domain}(f) = K : \Vert  f \Vert _{L^\infty(K)} < \infty  \} \\
C^k(K) &\triangleq \{\text{Domain}(f) = K : \Vert  f \Vert _{C^k(K)} < \infty  \} \end{align*}
when $K$ is compact,  $C^k(K) \subseteq L^\infty(K) \subseteq L^p(K)$, and an important fact is that $C^k(K)$ is dense in $L^p(K)$ under $L^p$ norm. Thus the universality over $C^k(K)$ in $C^k$-norm is stronger than the universality over $ L^\infty(K)$ in $ L^\infty$ norm, further stronger than the universality over $ L^p(K)$ in $ L^p$ norm.

\section{Proofs} \label{proof}
\subsection{Proofs for non-parametric case}
\subsubsection{Proofs for Theorem~\ref{decomposition1}}

\begin{defi}{Isotopies.}
 An isotopy between two diffeomorphisms $\phi_0, \phi_1 \in \text{Diff}_c^k\left(\mathbb{R}^d\right)$ is a $C^k$-map $H: [0,1]\times \mathbb{R}^d \to \mathbb{R}^d$ such that the mapping $h_t: \mathbb{R}^d \to \mathbb{R}^d$ defined by $h_t\left(x\right) = H\left(t,x\right)$ for all $t\in [0,1]$ satisfies $h_0 = \phi_0, h_1 = \phi_1$  and  $h_t \in \text{Diff}_c^k\left(\mathbb{R}^d\right)$ for all $t\in [0,1]$. It turns out that $t \to h_t$ is a continuous path in the group $\text{Diff}^k_c\left(\mathbb{R}^d\right)$ joining $\phi_0$ to $\phi_1$.  
\end{defi}
\begin{prop} (Proposition 1.2.1 in \cite{Banyaga1997The})
The group $\text{Diff}_c^k\left(\mathbb{R}^d\right)$ is connected. Moreover, the group $\text{Iso}^k\left(\mathbb{R}^d\right)$ of diffeomorphisms with compact supports which are isotopic to the identity map $I$ through isotopies coincide with $\text{Diff}_c^k\left(\mathbb{R}^d\right)$. Here the identity map I means $I\left(x\right) = x$ for all $x \in \mathbb{R}^d$.
\end{prop}
\begin{defi}{$\left(\delta,k\right)$-near-identity $C^k$-diffeomorphisms.}\label{near-id}
Let $B_{\delta,k}$ be the $C^k$-norm ball with radius $\delta$ and centered at identity map $I\left(\bm{x}\right) = \bm{x}$, that's to say, $B_{\delta,k} = \{ f \in \text{Diff}_c^k\left(\mathbb{R}^d\right) : \sup_{|\alpha| \le k}\Vert D^{|\alpha|} \left(f - I\right) \Vert _{L^{\infty}} < \delta \}$. 
A diffeomorphism $\phi \in \text{Diff}_c^k\left(\mathbb{R}^d\right)$ is said to be $\left(\delta,k\right)$-near-identity, if $\phi \in B_{\delta,k}$.
\end{defi}
\begin{lemma} (Lemma 2.1.8 in \cite{Banyaga1997The})
 \label{neariddecomposition}
For any diffeomorphism $f \in \text{Diff}_c^k\left(\mathbb{R}^d\right)$ and any $\delta>0$ , there exists a finite sequence of $\left(\delta,k\right)$-near-identity diffeomorphisms $g_1,\cdots,g_s$ such that $f = g_s\circ g_{s-1}\circ \cdots \circ g_1$.
\end{lemma}
\begin{proof}
Note that there exists an isotopy $h_t$ from $I$ to $f$ such that $h_0 = I$ and $h_1 = f$. We rewrite $f = h_1 = \left(h_1\circ h^{-1}_{\left(s-1\right)/s}\right) \circ \left(h_{\left(s-1\right)/s}\circ h_{\left(s-2\right)/s}^{-1}\right) \circ \cdots \circ \left(h_{1/s}\circ h_0^{-1}\right) $ and let $g_i = h_{i/s}\circ h_{\left(i-1\right)/s}^{-1}$, we can see that $f = g_s\circ g_{s-1}\circ \cdots \circ g_1$. Take $s$ large enough, we can make $h_{i/s}$ and $h_{\left(i-1\right)/s}$ close enough such that  $h_{i/s}\circ h_{\left(i-1\right)/s}^{-1}$ is $\left(\delta,k\right)$-near-identity.
\end{proof}
\begin{theorem}\label{firstdecomposition}
There exists a $\delta > 0$, such that for any $f\in \text{Diff}_c^1\left(\mathbb{R}^d\right)$ that is $\left(\delta,1\right)$-near-identity, f can be written as $g \circ h$ with $h\left(\bm{x},y\right) = \left(\bm{x}, \tilde h\left(\bm{x},y\right)\right)$ and $g\left(\bm{x},y\right) = \left(\tilde g\left(\bm{x},y\right),y\right)$ for $\bm{x} \in \mathbb{R}^{d-1}, y \in \mathbb{R}$. If $f \in \text{Diff}_c^k\left(\mathbb{R}^d\right)$, so are $g$ and $h$. Further more, $g$ satisfies $\left(\tilde\delta,1\right)$-near-identity for $\tilde \delta = \frac{\delta}{1-\delta} > 0$.
\end{theorem}
\begin{proof}
Let $\pi_i: \mathbb{R}^d \to \mathbb{R}$ denote the projection onto the $i^{th}$ coordinate. Suppose $f: \mathbb{R}^d \to \mathbb{R}^d$ is compactly supported and sufficiently $C^k$-close to the identity. Then for any point $\left(\bm{x},y\right) = \left(x_1,\cdots,x_{d-1}, y\right)$, the map $f_x: \mathbb{R} \to \mathbb{R}$ given by $f_{\bm{x}}\left(y\right) = \pi_n f\left(\bm{x},y\right)$ is a diffeomorphism: surjectivity follows from the fact that $f$ has compact support, which means $\lim_{y\to\pm\infty} f_{\bm{x}}\left(y\right) = \pm\infty$ and by the continuity of $f_{\bm{x}}$; injectivity follows from the fact that if $f_{\bm{x}}\left(y_1\right) = f_{\bm{x}}\left(y_2\right)$ for some $y_1 \neq y_2$, then $f_{\bm{x}}$ must has zero derivatives at some point $y \in \left(y_1, y_2\right)$, but the derivative of $f_{\bm{x}}$ respect to $y$ is near 1, a contradiction.

Now given $f$, define $h$ and $g: \mathbb{R}^{d-1}\times \mathbb{R} \to \mathbb{R}^{d-1}\times \mathbb{R}$ by
\begin{align*}
    h\left(\bm{x},y\right) &= \left(\bm{x}, f_{\bm{x}}\left(y\right)\right), \text{ and}\\
    g\left(\bm{x},y\right) &= \left(g_1\left(\bm{x},y\right), g_2\left(\bm{x},y\right) ,\cdots,g_{d-1}\left(\bm{x},y\right), y\right),
\end{align*}
where $g_i\left(\bm{x},y\right)  = \pi_i \left(f\left(\bm{x},f_{\bm{x}}^{-1}\left(y\right)\right)\right) \in \mathbb{R}.$ Obviously $g, h \in \text{Diff}_c^k\left(\mathbb{R}^d\right)$ given $f \in \text{Diff}_c^k\left(\mathbb{R}^d\right)$ and $f = g\circ h$. Also we observe that, $f$ is $\left(\delta,1\right)$-near-identity, thus
$$ \sup_{\bm{x},y} |\frac{\partial}{\partial y}f_{\bm{x}}\left(y\right) - 1| < \delta ,~\sup_{\bm{x},y} |\frac{\partial}{\partial x_i}f_{\bm{x}}\left(y\right)| < \delta,$$
$$ \sup_{\bm{x},y} |\frac{\partial}{\partial y} f_{\bm{x}}^{-1} \left(y\right) | = \sup_{\bm{x},y} |\frac{\partial}{\partial y}f_{\bm{x}}\left(y\right)|^{-1} < \frac{1}{1-\delta}. $$

Then we have
 \begin{align*}
  0=   \frac{d}{d x_i} y = \frac{d}{d x_i} \left(f_{\bm{x}}^{-1} \left( f_{\bm{x}} \left(y\right)\right)\right) 
= (\frac{\partial}{\partial x_i}f_{\bm{x}}^{-1}) \left( f_{\bm{x}} \left(y\right)\right) + (\frac{\partial}{\partial y}f_{\bm{x}}^{-1})\left( f_{\bm{x}} \left(y\right)\right)  \cdot \left(\frac{\partial}{\partial x_i} f_{\bm{x}}\left(y\right)\right),
\end{align*}
and further
\begin{align*}|\frac{\partial}{\partial x_i}f_{\bm{x}}^{-1}\left(y\right)| =  |\left(\frac{\partial}{\partial y}f_{\bm{x}}^{-1}\right)\left(y\right) \cdot \left(\left(\frac{\partial}{\partial x_i} f_{\bm{x}}\right)\left(f_{\bm{x}}^{-1}\left(y\right)\right)\right) | 
\le  \sup |\frac{\partial}{\partial y}f_{\bm{x}}^{-1}|\cdot \sup |\frac{\partial}{\partial x_i}f_{\bm{x}}| <  \frac{\delta}{1-\delta}.\end{align*}
If we denote $f_j\left(\bm{x},y\right) = \pi_j f\left( \bm{x},y\right)$, we get
\begin{align*}& |\frac{d}{d x_i} g_j\left(\bm{x},y\right) - \delta_{i,j}|  
= |\frac{\partial}{\partial x_i} f_j\left(\bm{x},f_{\bm{x}}^{-1}\left(y\right)\right) - \delta_{i,j} + \frac{\partial}{\partial y}\left(f_j\left(\bm{x},f_{\bm{x}}^{-1}\left(y\right)\right)\right) \cdot \frac{\partial}{\partial x_i}f_{\bm{x}}^{-1}\left(y\right) |
< \delta  + \delta\cdot\frac{\delta}{1-\delta} = \frac{\delta}{1-\delta},
\end{align*}
which proved that $g$ is $\left(\frac{\delta}{1-\delta},1\right)$-near-identity.
Here $\delta_{i,j}$ are the kronecker symbols and we notice $|\frac{\partial}{\partial x_i} f_j - \delta_{i,j} | < \delta$.
\end{proof}

\begin{cor}\label{totaldecomposition}
There is a $0 < \delta < \frac{1}{d-1} $ such that for any $f\in \text{Diff}_c^k\left(\mathbb{R}^d\right)$ is $\left(\delta,1\right)$-near-identity, f can be written as $f_1 \circ f_2 \circ \cdots f_n$, $f_i \in S^{1,d}_c$ with $f_i\left(\bm{x}\right) = \left(x_1, x_2 ,\cdots, x_{i-1},\tilde f_i\left(\bm{x}\right),x_{i+1},\cdots, x_d\right) $ for $\bm{x} = \left(x_1,\cdots,x_d\right) \in \mathbb{R}^d$. If $f$ is in $\text{Diff}_c^k\left(\mathbb{R}^d\right)$ for some $k > 1$, then so are $f_i$. 
\end{cor}
\begin{proof}
By taking $\delta$ small enough, we can make $\tilde \delta = \delta / \left(1-\delta\right)$ small enough, thus the $g$ in Thm.~\ref{firstdecomposition} can be further decomposed. A simple observation shows that if $f$ already preserves some coordinates, then so does $g$.  If we set $\delta_i = \frac{\delta_{i-1}}{1 - \delta_{i-1}}$ and $\delta_1 = \delta < \frac{1}{n-1}$, by noticing $\frac{1}{\delta_{n-1}} = \frac{1}{\delta_1} - n+2 > 1$, we can make $\delta_{n-1}$ small enough, thus the decomposition can always be continued until all the coordinates have been decomposed.
\end{proof}

\begin{proof}[Proof of Thm.~\ref{decomposition1}]

It can be immediately proved by Lem.~\ref{neariddecomposition}  and Cor.~\ref{totaldecomposition} .
\end{proof}

\subsubsection{Proof for Theorem~\ref{compos-approx}}\label{proof-compos}
\begin{proof}

Take any positive number $1 > \tilde\epsilon > 0$ and compact set $K\in \mathbb{R}^d$. Put $r \triangleq \max_{\bm{x} \in K} \Vert f_1\left(\bm{x}\right)\Vert $ and $K' \triangleq \{ \bm{x}\in\mathbb{R}^d: \Vert \bm{x}\Vert  \le r+1\}.$ Let $g_2 \in G$ satisfying
$$ \Vert f_2 - g_2\Vert _{C^k\left(K'\right)} < \tilde\epsilon.$$
Since any continuous map is uniformly continuous on a compact set, we take a positive number $\delta > 0$ such that for any $\bm{x},\bm{y}\in K'$ with $|\bm{x} - \bm{y}| < \delta$,
$$\sup_{|\alpha| \le k} \Vert D^\alpha f_2\left(\bm{x}\right) - D^\alpha f_2\left(\bm{y}\right)\Vert  < \tilde\epsilon.$$
From the assumption, we can take $g_1 \in G$ satisfying 
$$\Vert f_1 - g_1\Vert _{C^k\left(K\right)} < \min\{1,\delta\}.$$
Then it is clear that $f_1\left(K\right) \subseteq K'$ by the definition of $K'$. Moreover, we have $g_1\left(K\right) \subseteq K'$. In fact, we have
$$\Vert g_1\left(\bm{y}\right)\Vert  \le \sup_{\bm{x} \in K} \Vert  f_1\left(\bm{x}\right) - g_1\left(\bm{x}\right)\Vert  + |F_2 \left(\bm{y}\right)| \le 1+r
$$
for any $\bm{y} \in K'$.

Then for any $\bm{x} \in K$, we have
\begin{align*}&\Vert f_2 \circ f_1 - g_2 \circ g_1\Vert \\ \le &\Vert f_2\circ f_1 - f_2\circ g_1\Vert  + \Vert f_2\circ g_1 - g_2\circ g_1\Vert  < 2 \tilde \epsilon. \end{align*}
Now let's consider the cumulative error for derivatives. We have 
\begin{align*}
    &
    \Vert D \left(f_2 \circ f_1\right) - D\left(g_2\circ g_1\right)\Vert\\
    \le& \Vert D \left(f_2 \circ f_1\right)  - D\left(f_2\circ g_1\right)\Vert  + 
    \Vert D \left(f_2 \circ g_1\right)  - D\left(g_2\circ g_1\right)\Vert \\
    =& \Vert \left(D f_2\right) \circ f_1\cdot D f_1 - \left(D f_2\right) \circ g_1 \cdot D g_1\Vert  +  \Vert \left(D f_2\right) \circ g_1 \cdot D g_1 - \left(D g_2\right) \circ g_1 \cdot D g_1\Vert  \\
    \le& \Vert \left(D f_2\right) \circ f_1 \cdot D \left(f_1- g_1\right)\Vert 
    + \Vert \left(D\left(f_2-g_2\right)\right)\circ g_1  \cdot Dg_1\Vert  +  \Vert \left(\left(Df_2\right)\circ f_1 -  \left(D f_2\right) \circ g_1\right) \cdot D g_1\Vert  \\
    <& \left(\Vert Df_2\Vert  + 2\Vert Dg_1\Vert \right) \tilde \epsilon \\
    <& C\left(f_1, f_2\right) \tilde \epsilon
\end{align*}
by noticing that
\[
    \Vert Dg_1\left(\bm{x}\right)\Vert  \le \Vert Df_1\left(\bm{x}\right)\Vert  + \tilde\epsilon \le \Vert Df_1\left(\bm{x}\right)\Vert  + 1.
\]

Higher order derivatives can be estimated following the same procedure with more complex computations and reusing of triangular inequality. We can finally arrive at
$$\Vert f_2\circ f_1\left(\bm{x}\right) - g_2 \circ g_1\left(\bm{x}\right)\Vert _{W^{\infty,k}\left(K\right)} < \tilde C\left(f_1, f_2\right) \tilde \epsilon$$
with $\tilde C\left(f_1, f_2\right)$ only depends on $f_1$ and $f_2$ and their derivatives, doesn't depend on $\tilde \epsilon$ because $f_1, f_2$ are compactly supported, which means they have finite high order derivatives over $\mathbb{R}^d$.

Thus we take $\tilde \epsilon = \frac{\epsilon}{\tilde C\left(f_1,f_2\right)}$, then $\Vert f_2\circ f_1\left(\bm{x}\right) - g_2 \circ g_1\left(\bm{x}\right)\Vert _{W^{\infty,k}\left(K\right)} \le \epsilon$, and then finished our proof.
\end{proof}
\subsubsection{Proof for Theorem~\ref{single-approx} }\label{single-approx-chapter}
\begin{proof}
 Note that
 \begin{equation*}
    \iota_d \circ \tau \circ \pi_d\left(x_1,x_2,\cdots,x_d,0\right) = \left(x_1,x_2,\cdots,\tau_d\left(\bm{x}\right),0\right).
 \end{equation*}
 This can be decomposed into three small steps: 
\begin{align*}
    \left(x_1,x_2,\cdots,x_{d},0\right) \overset{\phi_1}{\to} 
    \left(x_1,x_2,\cdots,x_{d},\tau_d\left(\bm{x}\right)\right)\overset{\phi_2}{\to}
    \left(x_1,x_2,\cdots, \tau_d\left(\bm{x}\right), x_{d}\right) \overset{\phi_3}{\to} \left(x_1,x_2,\cdots, \tau_d\left(\bm{x}\right), 0\right).
\end{align*}
Next let us approximate $\phi_1,\phi_2,\phi_3$ using the elements in $\mathcal{G}\text{-INN}_{d+1}$. By definition, $\phi_1$ can be written as $\Phi_{d+1,d,\sigma,t}$ with $\sigma$ to be any function, $t\left(\bm{x}\right) = \tau_d\left(\bm{x}\right)$. By assumption, $\mathcal{H}$ has $C^k$-universality for $t$, thus we know $\mathcal{G}\text{-INN}_{d+1}$ has universality for $\phi_1$. $\phi_2$ is just a permutation which is already in our layers. $\phi_3$ can be written as
$\Phi_{d+1,d,\sigma,t}$ with $\sigma=0$,  $t\left(\bm{x}\right) = \tau_d^{-1}\left(\bm{x}\right)$. Here $\tau_d^{-1}\left(\bm{x}\right)$ is the inverse of $\tau_d\left(\bm{x}\right)$ w.r.t. $x_d$ because $\tau_d\left(\bm{x}\right)$ is a monotonic function w.r.t. $x_d$. Thus, we claim that $\mathcal{G}\text{-INN}_{d+1}$ has universality for $\phi_3$. By Thm.~\ref{compos-approx}, we know that $\iota_d \circ \tau \circ \pi_d$ can be arbitrarily approximated by $ \mathcal{G}\text{-INN}_{d+1}$.

Thus, for any $\epsilon > 0$, there exists a $\tilde \tau \in \mathcal{G}\text{-INN}_{d+1}$ such that
\[
    \Vert \iota_d\circ\tau\circ\pi_d - \tilde \tau\Vert _{C^{k}\left(K\times\mathbb{R}\right)} < \epsilon,
\]
and furthermore,
\begin{align*}
    &\phantom{{}={}}
    \Vert \tau - \pi_d\circ\tilde \tau\circ\iota_d\Vert _{C^{k}\left(K\right)}=\Vert \left(\pi_d\circ\iota_d\right)\circ\tau\circ\left(\pi_d\circ\iota_d\right) - \pi_d\circ\tilde \tau\circ\iota_d\Vert _{C^{k}\left(K\right)}
    \le \Vert \pi_d\Vert \cdot \Vert \iota_d\circ\tau\circ\pi_d - \tilde \tau\Vert _{C^{k}\left(K\times\mathbb{R}\right)} \cdot  \Vert \iota_d\Vert < \epsilon.
\end{align*}
\end{proof}

\subsection{Proofs for parametric case}
\subsubsection{Proofs for Theorem~\ref{decompositionpara}}

\begin{defi}{Isotopies along $\text{Diff}_c^{k,m,d}$.}
 An isotopy between two diffeomorphisms $\phi_0, \phi_1 \in \text{Diff}_c^{k,m,d}$ is a $C^k$-map $H: [0,1]\times \mathbb{R}^{m+d} \to \mathbb{R}^{m+d}$ such that the mapping $h_t: \mathbb{R}^d \to \mathbb{R}^d$ defined by $h_t\left(x\right) = H\left(t,x\right)$ for all $t\in [0,1]$, $h_0 = \phi_0, h_1 = \phi_1$  and  $h_t \in \text{Diff}_c^{k,m,d}$ for all $t\in [0,1]$. It turns out that $t \to h_t$ is a continuous path in the group $\text{Diff}_c^{k,m,d}$ joining $\phi_0$ to $\phi_1$. 
\end{defi}

\begin{theorem}
The group $\text{Diff}_c^{k,m,d}$ is connected. Moreover, the group of diffeomorphisms with compact supports which are isotopic to the identity map $I$ through compactly supported isotopies  coincide with $\text{Diff}_c^{k,m,d}$. Here the identity map $I$ means $I\left(x\right) = x$ for all $x \in \mathbb{R}^{m+d}$.
\end{theorem}
\begin{proof}
For arbitrary $f \in \text{Diff}_c^{k,m,d}$, we have a compactly supported $C^k$-isotopy $$H: \mathbb{R}^{m+d} \times [\epsilon,1] \to \mathbb{R}^{m+d},$$ given by the proportionately contraction $$H\left(\left(\bm{y},\bm{x}\right),t\right)\triangleq t f\left( \bm{y}/t,{\bm{x}}/t\right), \bm{y}\in \mathbb{R}^m, \bm{x} \in \mathbb{R}^d.$$ It is obvious that $H\left(\left(\bm{y},\bm{x}\right),t\right)$ always lie in $\text{Diff}_c^{k,m,d}$, the first $m$-coordinate corresponding to $y$ are fixed.
By choosing $\epsilon$ small enough, we can achieve $h_{\epsilon}\left(\bm{y},\bm{x}\right) = H\left(\left(\bm{y},\bm{x}\right),\epsilon\right) \in \left({\text{Diff}_{c}^{0,m,d}}\right)_0$, the $C^0$-connected neighborhood of $I$.

Thus, there is a  prolongation of $H:\mathbb{R}^{m+d} \times [0,1] \to \mathbb{R}^{m+d}$ that is a compactly supported $C^0$-isotopy from $I$ to $f$ with $H_t = I, \text{ for all } t \in [0,\epsilon/2]$. Since compactly supported $C^k$-isotopy space with first $m$-coordinate fixed is dense in compactly supported $C^0$-isotopy space with first $m$-coordinate fixed, and since $H$ is already $C^k$ on $\mathbb{R}^{m+d} \times  \left([0,\epsilon/2]\cup [\epsilon,1]\right)$, we can always find a compactly supported $C^k$-isotopy $\tilde H$ from $I$ to $f$, and thus $f$ is isotopic to the identity $I$ in $\text{Diff}_c^{k,m,d}$.
\end{proof}

Follow the same proof in Lem.~\ref{neariddecomposition}, we can get the following lemma.

\begin{theorem}
 \label{neariddecompositionpara}
For any diffeomorphism $f \in \text{Diff}_c^{k,m,d}$ and any $\delta>0$ of the identity, there exists a finite sequence of $\left(\delta,k\right)$-near-identity diffeomorphisms $g_1,\cdots,g_s \in \text{Diff}_c^{k,m,d}$ such  that $f = g_s\circ g_{s-1}\circ \cdots \circ g_1$.
\end{theorem}

We mark that Thm.~\ref{firstdecomposition} and Cor.~\ref{totaldecomposition} need no modification and can be directly applied to here.

\begin{proof} of Thm.~\ref{decompositionpara}.

It is immediately proved by Lem.~\ref{neariddecompositionpara}  and Cor.~\ref{totaldecomposition} by replac int $d$ in Cor.~\ref{totaldecomposition} with $ m+d$.
\end{proof}

\section{Dimension Eliminating layer for inverse model}\label{eliminating}
Consider a trained affine coupling flow with $z$ padding zeros:
\begin{equation*}
    \big(
    \bm{y},
    \bm{x},
    {\bm{0}}
    \big)
    \xmapsto{\text{forward}}
    \left( 
    \bm{y},
    \hat f\left(\bm{y},\bm{x}\right),
    \theta\left(\bm{y},\bm{x}\right)
    \right),
\end{equation*}
here $\hat f\left(\bm{y},\bm{x}\right)\in \mathbb{R}^d$ is a trained surrogate of the parametric diffeomorphism $f\left(\bm{y},\bm{x}\right) \in \mathbb{R}^d$. $\bm{y}\in \mathbb{R}^m, \bm{x}\in \mathbb{R}^d. \theta \in \mathbb{R}^z$ is something we don't need to consider in the forward map, but must consider when doing inverse map.

Given any point $\left(\bm{y}, \hat{\bm{x}}\right)$  in output domain, we are not able to directly find the cooresponding input point $\left(\bm{y}, \bm{x}\right)$ satisfying $\hat{\bm{x}} = \hat f\left(\bm{y},\bm{x}\right)$, because $\theta\left(\bm{y},\bm{x}\right)$ is unknown. 

To handle this problem, we can make $\theta\left(\bm{y},\bm{x}\right)$ to be known. The simplest way is to transform it to zeros. We can train such an eliminating maps:
\begin{equation*}
    \left( 
    \bm{y},
    \hat f\left(\bm{y},\bm{x}\right),
    \theta\left(\bm{y},\bm{x}\right)
    \right)
    \xmapsto{\text{eliminating}}
    \left( 
    \bm{y},
    \hat f\left(\bm{y},\bm{x}\right),
    o\left(\bm{y},\bm{x}\right)
    \right),
\end{equation*}
here $o\left(\bm{y},\bm{x}\right) \approx {\bm{0}}$ for all $\bm{y},\bm{x}$ in a certain compact set. Note that the eliminating map is always trainable in the scope of affine coupling structures: let $\hat{\bm{x}} = \hat f\left(\bm{y},\bm{x}\right)$, $\left(\bm{y},\bm{x}\right) \mapsto \left(\bm{y},\hat{\bm{x}}\right) $ is approx. invertible, thus $\theta\left(\bm{y},\bm{x}\right) \approx \hat\theta\left(\bm{y},\hat{\bm{x}}\right)$ for some $\tilde \theta$, a function of $y$ and $\hat {\bm x}$. By the structure of the affine coupling layers, 
\begin{equation*}
    \left( 
    \bm{y},
    \hat {\bm{x}},
    \hat\theta\left(\bm{y},\hat{\bm{x}}\right)
    \right)
    \longmapsto
    \left( 
    \bm{y},
    \hat {\bm{x}},
    {\bm{0}}
    \right)
\end{equation*}
can always be approximated.

After the eliminating, we finally get a map: 
\begin{equation*}
    \left(
    \bm{y},
    \bm{x},
    {\bm{0}}
    \right)
    \mapsto
    \left( 
    \bm{y},
    \hat {\bm{x}},
    o\left(\bm{y},\bm{x}\right)
    \right),
\end{equation*}
with $o\left(\bm{y},\bm{x}\right)\approx 0$ is neglectable. Thus given any $\left(\bm{y},
\hat {\bm{x}}\right)$, we can easily compute the original point  $\left(\bm{y},\bm{x}\right)$ because
such a map is characterized by INN, and thus easy to compute the inverse.

\begin{table}[ht]\small
    \centering
    \renewcommand{\arraystretch}{1.1}
    \setlength\tabcolsep{2pt}
    \caption{When dimensionality of context is 5: 
        \#hidden-layers and  \#hidden-nodes represent the  number of hidden layers and nodes for state network in \texttt{Para\_CFlow} and base network in other models, respectively. 
        \#flows is the number of modules in a flow-based model.
        \#param is the total number of trainable parameters in a neural surrogate model.
    }
    \label{Tab:Setting:Dim:5}
    \begin{tabular}{lcccc}
    \toprule
    Method      & \#hidden-layers & \#hidden-nodes & \#flows & \#parameters \\
    \midrule
    Para-CFlow  & 1               & 64             & 3       & 1428         \\
    MLP         & 2               & 32             & 0       & 1313         \\
    MLP-Ascend & 2               & 32             & 0       & 1481         \\
    Resnet      & 2               & 32             & 0       & 1346  \\
    \bottomrule
    \end{tabular}
\end{table}

\begin{table}[ht]\small
    \centering
    \renewcommand{\arraystretch}{1.1}
    \setlength\tabcolsep{2pt}
    \caption{When dimensionality of context is 10: 
        The meaning of the header is the same as in Tab.~\ref{Tab:Setting:Dim:5}.
    }
    \label{Tab:Setting:Dim:10}
    \begin{tabular}{lcccc}
        \toprule
        Method      & \#hidden-layers & \#hidden-nodes & \#flows & \#parameters \\
        \midrule
        Para-CFlow  & 0               & 128            & 3       & 5337         \\
        MLP         & 2               & 64             & 0       & 4993         \\
        MLP-Ascend & 2               & 64             & 0       & 5669         \\
        Resnet      & 2               & 64             & 0       & 5058 \\
        \bottomrule
    \end{tabular}
\end{table}

\begin{table}[ht]\small
    \centering
    \renewcommand{\arraystretch}{1.1}
    \setlength\tabcolsep{2pt}
    \caption{When dimensionality of context is 20: 
        The meaning of the header is the same as in Tab.~\ref{Tab:Setting:Dim:5}.
    }
    \label{Tab:Setting:Dim:20}
    \begin{tabular}{lcccc}
        \toprule
        Method      & \#hidden-layers & \#hidden-nodes & \#flows & \#parameters \\
        \midrule
        Para-CFlow  & 1               & 64             & 3       & 12723        \\
        MLP         & 3               & 64             & 0       & 9793         \\
        MLP-Ascend & 3               & 64             & 0       & 11469        \\
        Resnet      & 3               & 64             & 0       & 9922    \\
        \bottomrule
    \end{tabular}
\end{table}

\section{Experiments}\label{Sec:Appedix:Exp}
In our experiments, 
we use 3 well-known benchmark functions: \texttt{Ackley}, \texttt{Trid} and \texttt{Rastrigin} \cite{benchmark_functions-web}.
It is noted that the original benchmark tasks are \textbf{contextless}.
Here,
we simply set the first $d_c$ dimensions as the context vector and leave the last one dimension as the action to be optimized.
For our constructed optimization problems,
\texttt{Ackley} has the action dimension that is coupled with all context dimensions,
\texttt{Trid}'s action is coupled with just one context dimension,
and all the dimensions of \texttt{Rastrigin} are independent with each other.
To construct context-dominating tasks where the reward depends much more on the context than the action, we set $d_c\gg 1$.

\subsection{Experimental settings}\label{Sec:Appendix:Exp:Settings}
In the training of all the surrogate models,
we consistently set the number of batch size to 64,
number of the epochs to 200, and the learning rate to 0.01.
For cases when the dimensionality of context is 5, 10 and 20,
we implement each of the aforementioned surrogate models in Sec.~\ref{Sec:Exp:BO} with similar sizes (number of trainable parameters)
and the corresponding specifications are posted as in Tab.\ref{Tab:Setting:Dim:5}, Tab.~\ref{Tab:Setting:Dim:10}, and Tab.~\ref{Tab:Setting:Dim:20}, respectively.

\subsection{Cumulative regrets}\label{Sec:Appedix:Exp:CumulativeRegrets}
The cumulative regrets of BO using Thompson's sampling with different neural surrogate models for context of dimensionality 5, 10 and 20 are shown in Fig.~\ref{Fig:Appendix:Exp:CumulativeRegrets:Thompson}.

\begin{figure*}[ht]
    \centering
    \includegraphics[width=.9\linewidth]{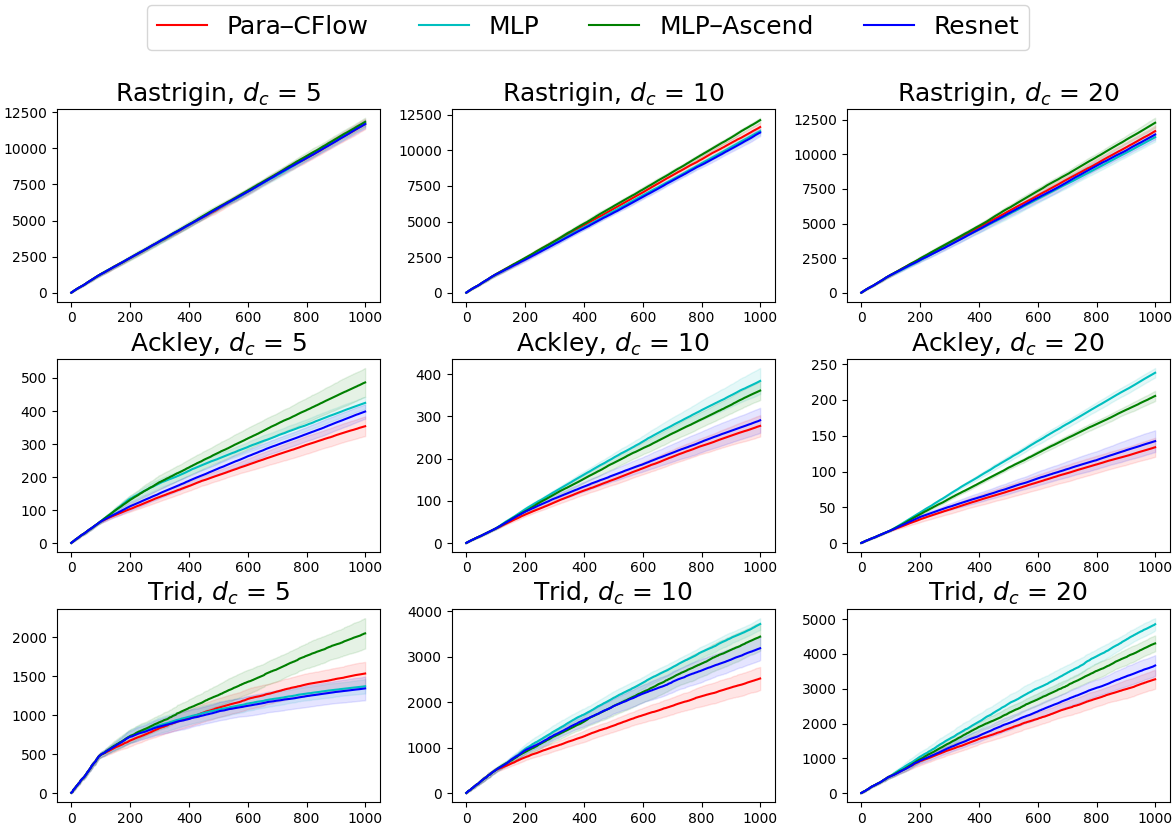}
    \vspace{-2ex}
    \caption{The mean and standard deviation of the cumulative regret under 10 independent trials using \textit{Thompson's sampling} with different neural surrogate models for context of dimensionality 5, 10, and 20 on \texttt{Rastrigin}, \texttt{Ackley} and \texttt{Trid}, respectively.}
    \label{Fig:Appendix:Exp:CumulativeRegrets:Thompson}
\end{figure*}

\end{document}